\documentclass[ journal, comsoc ]{IEEEtran}
\usepackage[T1]{fontenc}

\ifCLASSINFOpdf
\else
\fi
\usepackage{amsmath}
\usepackage{amsthm}
\usepackage[cmintegrals]{newtxmath}

\usepackage{amssymb}
\usepackage{bm}
\usepackage{bbm}
\usepackage{subfigure}
\usepackage{graphicx}
\usepackage{calc}
\usepackage{epstopdf}
\newtheorem{theorem}{Theorem}
\newtheorem{lemma}{Lemma}

\newtheorem{proposition}{Proposition}

\newtheorem{assumption}{Assumption}
\usepackage{verbatim}
\usepackage[ruled,vlined]{algorithm2e}
\usepackage{color}
\usepackage{array}
\usepackage{cite}
 
\usepackage{xcolor}
\usepackage{floatrow}
\usepackage{lipsum}

\usepackage{multicol}
\usepackage{multirow} 
\usepackage{algorithmic}

\begin{document}
\title{Decentralized Multi-Task Learning Based on Extreme Learning Machines}

\author{Yu~Ye,~\IEEEmembership{Student~Member,~IEEE,}
	Ming~Xiao,~\IEEEmembership{Senior~Member,~IEEE,}
	and~Mikael~Skoglund,~\IEEEmembership{Fellow,~IEEE }
	
	\thanks{ Yu Ye, Ming Xiao and Mikael Skoglund are with the School of Electrical Engineering and Computer Science, Royal Institute of Technology (KTH), Stockholm, Sweden (email: yu9@kth.se, mingx@kth.se, skoglund@kth.se).}
}
  
\maketitle
 
\begin{abstract}
 
In multi-task learning (MTL), related tasks learn jointly to improve generalization performance. To exploit the high learning speed of extreme learning machines (ELMs), we apply the ELM framework to the MTL problem, where the output weights of ELMs for all the tasks are learned collaboratively. 
We first present the ELM based MTL problem in the centralized setting, which is solved by the proposed MTL-ELM algorithm. 
Due to the fact that many data sets of different tasks are geo-distributed, decentralized machine learning is studied. 
We formulate the decentralized MTL problem based on ELM as majorized multi-block optimization with coupled bi-convex objective functions. To solve the problem, we propose the DMTL-ELM algorithm, which is a hybrid Jacobian and Gauss-Seidel Proximal multi-block alternating direction method of multipliers (ADMM). Further, to reduce the computation load of DMTL-ELM, DMTL-ELM with first-order approximation (FO-DMTL-ELM) is presented. Theoretical analysis shows that the convergence to the stationary point of DMTL-ELM and FO-DMTL-ELM can be guaranteed conditionally. Through simulations, we demonstrate the convergence of proposed MTL-ELM, DMTL-ELM and FO-DMTL-ELM algorithms, and also show that they can outperform existing MTL methods. Moreover by adjusting the dimension of hidden feature space, there exists a trade-off between communication load and learning accuracy for DMTL-ELM. 
 
\end{abstract}

\IEEEpeerreviewmaketitle

\section{Introduction}

 Machine learning usually requires a large amount of training samples to obtain an accurate learner, e.g., the deep neural network with a tremendous number of model parameters. However, in many real-word applications, it is hard to collect enough samples for training. One efficient solution when in shortage of data is Multi-Task Learning (MTL) \cite{caruana1997multitask}, of which the main goal is to improve generalization performance by leveraging the domain-specific information contained in the training samples of tasks. In MTL, the learning tasks are assumed to be related to each other, and it is found that learning them jointly can achieve better generalization performance rather than learning individually.

 The MTL approach seeks to learn the relationship of model parameters among tasks, and it has been extensively investigated \cite{evgeniou2004regularized,zhou2016flexible,argyriou2007multi,argyriou2008convex,amit2007uncovering,kang2011learning,kumar2012learning,romera2012exploiting,zhang2012convex,li2018better,chen2013convex}. In \cite{evgeniou2004regularized}, it is assumed that the model parameters of task $t$ can be represented as $\bm{w}_t = \bm{w}_0 + \bm{v}_t$, where $\bm{w}_0$ is common across tasks and $\bm{v}_t$ is task specific deviation. 
 The clustered multi-task learning (CMTL) is studied in \cite{zhou2016flexible}, which learned the clustered structure of tasks by identifying representative ones. While in \cite{argyriou2007multi,argyriou2008convex,amit2007uncovering,kang2011learning,romera2012exploiting,kumar2012learning}, the multi-task feature learning (MTFL) is considered, based on the assumption that the model parameters of task $t$ can be given by $\bm{w}_t=\bm{U}\bm{a}_t$. The feature space $\bm{U}$ is common for all the tasks while $\bm{a}_t$ is task specified. 
 In \cite{argyriou2007multi,argyriou2008convex} and \cite{amit2007uncovering }, $\ell_1$-norm penalty and $\ell_2$-norm penalty are considered in the objective functions, respectively.
 In \cite{kang2011learning}, the tasks are assumed to be in disjoint groups. 
 While in \cite{kumar2012learning},
 the tasks in different groups are allowed to overlap with each other in bases. 
 In \cite{chen2013convex}, the model parameters for task $t$ is assumed to be $\bm{w}_t=\bm{u}_t+\varTheta \bm{v}_t$, where $\varTheta$ captures the correlation among tasks.
 Combining the assumptions in \cite{evgeniou2004regularized} and \cite{argyriou2007multi}, the shared features and shared parameters are considered simultaneously through assuming $\bm{w}_t = \bm{U}(\bm{a}_0+ \bm{a}_t)$. In this way both feature relatedness and model relatedness can be modeled.  
 
 In \cite{evgeniou2004regularized,zhou2016flexible,argyriou2007multi,argyriou2008convex,amit2007uncovering,kang2011learning,kumar2012learning,romera2012exploiting,zhang2012convex,li2018better,chen2013convex}, the data from all the tasks are in a central location. However in many real-world applications, the data of different tasks may be separately located in different machines. Thus the centralized methods may be inefficient.  
MTL on distributed networks is an active research area and has attracted lots of research interests.
 Recently, various distributed MTL problems have been studied \cite{nassif2016proximal,li2017distributed,li2018distributed,zhang2018distributed}. Utilizing the equivalent convex optimization formulation in \cite{argyriou2008convex}, which characterizes the correlation between model parameters $\bm{w}_t$ by a matrix $\bm{\Omega}$, the distributed multi-task relationship learning is studied in \cite{liu2017distributed,smith2017federated,wang2018distributed}. 
 In \cite{wang2016distributed1}, a communication-efficient estimator based on the debiased lasso is presented. Reference \cite{hua2017distributed} learned a shared predictive structure for tasks by extending \cite{chen2013convex} to a distributed setting.  
 Following the assumptions of MTFL in \cite{argyriou2007multi}, two communication-efficient subspace pursuit algorithms are provided in \cite{wang2016distributed2}. However, the strategies in \cite{wang2016distributed2} can only be implemented in master-salve network structures. 
 
 For the centralized MTL problems studied in \cite{argyriou2008convex,zhang2012convex,kang2011learning,kumar2012learning,argyriou2007multi}, \textit{alternating optimization} (AO) method is utilized to obtain the solutions. In each iteration of the AO method, one variable is optimized while another is fixed \cite{beck2015convergence}. Yet in distributed MTL (DMTL), objectives and variables may be separated and the AO method is invalid. Hence one of the challenges for DMTL is to design distributed optimization algorithms. Confronted with this problem, the alternating direction method of multipliers (ADMM) \cite{boyd2011distributed} has been shown as an efficient solution, which was used to solve a global consensus problem with regularization in \cite{luo2017distributed}. Extending the \textit{global consensus problem} \cite{boyd2011distributed} to multi-block convex optimization with separable objectives and linear constraints, a Jacobian Proximal ADMM is proposed in \cite{Deng2017}, where the variables are updated in a parallel manner. The ADMM based algorithms for multi-block convex optimizations with coupled objectives are provided in \cite{xu2018hybrid,Gao2017,cui2016convergence,beck2013convergence,liu2018proximal,Banjac2018}. A Gauss-Seidel type ADMM with first-order approximation is provided to solve the general multi-block optimization with coupled objective function in \cite{Gao2017}. Different from the Jacobian type method, the variables are sequentially optimized for Gauss-Seidel method. In \cite{hua2017distributed}, the optimization method that integrates block coordinate descent method (BCD) with the inexact ADMM is utilized for distributed learning. While a BCD for regularized multi-convex optimization is considered in \cite{Xu2013}. Furthermore, the ADMM in nonconvex and nonsmooth optimization is studied in \cite{Wang2019}.  
 
 Since the tasks are trained together both for centralized and distributed MTL, the model will become complicated and hence significantly reduce the training speed, especially for neural networks (NN) with multiple layers. To balance the training speed and the generalization performance of MTL and DMTL, the extreme learning machine (ELM) for single-hidden layer feed-forward neural networks (SLFNs) can be utilized for the basic tasks.
 Because there is only one hidden layer in ELM, and the hidden nodes can be chosen randomly, the output weights of SLFNs can be analytically determined. According to the analysis and experiments in \cite{huang2006extreme}, the ELM can provide good generalization performance in most learning cases and learn thousands of times faster than conventional popular learning algorithms for feed-forward neural networks. The distributed extreme learning machine is studied in \cite{luo2017distributed,bi2015distributed}, where one single task is distributed to several workers in parallel.

Based on above observation, we will study the MTL approach with ELM implementation by considering both centralized and decentralized scenarios. Different from \cite{evgeniou2004regularized,zhou2016flexible,argyriou2007multi,argyriou2008convex,amit2007uncovering,kang2011learning,kumar2012learning,romera2012exploiting,zhang2012convex,li2018better,chen2013convex,nassif2016proximal,li2017distributed,li2018distributed,zhang2018distributed,liu2017distributed,smith2017federated,wang2018distributed,wang2016distributed1,wang2016distributed2,hua2017distributed}, the input data for tasks is first randomly mapped to a hidden feature space. In our MTL method \cite{caruana1997multitask}, the weights of ELM for related tasks are correlated. Hence instead of learning the weights of ELM individually, the tasks can leverage training data sets from each other to obtain model parameters with better generalization performance. This can be achieved for the centralized case since all the data sets are available in a single node. However, when tasks are located separately, transmitting training datasets across tasks is costly, as well as may cause the privacy leakage problem. Thus we utilize the decentralized optimization method to update each local weights through exchanging intermediate information. To the best of our knowledge, MTL or DMTL with ELM implementation has not been studied before. The main contributions of this paper are listed as follows.
 \begin{itemize}
 	\item We study the centralized multi-task learning machine based on ELM (MTL-ELM). Then based on \textit{Alternating Optimization} (AO) method, the MTL-ELM algorithm is proposed to obtain shared subspace across tasks and the task-specified weights. We show that the proposed MTL-ELM algorithm can be guaranteed to converge to a stationary point of the problem;
 	\item The MTL-ELM problem is further considered for the decentralized scenario. We extend the centralized learning problem by introducing sharing variables for localized task objectives as well as a consensus constraint. Then the algorithm DMTL-ELM of hybrid Jacobian and Gauss-Seidel Proximal multi-block alternating direction method of multipliers (ADMM) is proposed to solve the decentralized learning problem; 
 	\item We generalize the decentralized learning problem as majorized multi-block optimization problem with coupled bi-convex objective functions, which was not been studied before, to our best knowledge. Through theoretical analysis, we prove that the DMTL-ELM converges to a stationary point when algorithm parameters meet specific conditions. 
 	\item To reduce the computation load of DMTL-ELM, we propose the FO-DMTL-ELM algorithm, which utilizes the first-order approximation in the update process. The convergence of FO-DMTL-ELM to stationary points are proved. 
 	\item By simulations, we show the convergence of algorithms MTL-ELM, DMTL-ELM and FO-DMTL-ELM. By experiments on real-world datasets, we also show that MTL-ELM, DMTL-ELM and FO-DMTL-ELM can outperform state-of-the-art MTL approaches in terms of the generalization performance. 	
 	Moreover, the trade-off between communication load and learning accuracy for DMTL-ELM and FO-DMTL-ELM are shown through experiments.
 \end{itemize}

The rest of this paper is organized as follows. In Section I, we formulate the ELM based MTL problem in a centralized setting, which is solved by the proposed MTL-ELM algorithm. Then the DMTL problem with ELM implementation is presented in Section III. Two fully decentralized ADMM approach DMTL-ELM and FO-DTML-ELM are provided, as well as the analysis of convergence. Numerical results are given in Section IV to show the generalization performance of the proposed approaches. Finally, we draw conclusions in Section V.      
 
\section{Multi-task learning based on ELM}
 In this section, we will first present the principles of ELM. Then we shall integrate the ELM with centralized MTL by modifying the models of ELM.
\subsection{Principles of ELM}

ELM refers to a class of single-hidden-layer feed-forward neural networks (FNNs) shown in Fig. 1 (a), where the hidden layer needs not be tuned. The output of an ELM network with $L$ hidden nodes is formulated as  
\begin{equation}
Y_L(\bm{X})=\sum\nolimits_{l=1}^{L}\beta_lh_l(\bm{X})=\bm{h}(\bm{X})\bm{\beta},
\end{equation}
where $\bm{\beta}= \left[\beta_1,...,\beta_L  \right]\in \mathbb{R}^{L}$ is the output weights. $\bm{h}(\bm{X})= \left[h_1(\bm{X}),...,h_L(\bm{X}) \right]$ is the feature mapping $\bm{h}: \mathbb{R}^n\to  \mathbb{R}^L$, which maps input variable $\bm{X}\in \mathbb{R}^n$ to $L$-dimension hidden-layer feature space. The component $h_l(\bm{X})=G(\bm{w}_l,b_l,\bm{X}): \mathbb{R}\to\mathbb{R}$ denotes the impulse function of the $l$-th hidden node given by activation function $G(\bm{w}_l,b_l,\bm{X})=g(\bm{w}_l \bm{X}+b_l)$ (e.g. sigmoid function: $g(x)=1/(1+\exp(-x))$, where the weights $\bm{w}_l$ and bias $b_l$ are randomly generated according to any continuous probability distribution \cite{huang2012extreme}. The output weights $\bm{\beta}$ need to be learned from a training data set $D =  \{ (\bm{X}_i,T_i  ) | \bm{X}\in\mathbb{R}^n,T_i\in \mathbb{R}^d,i=1,..., N  \}$ by solving
\begin{equation}
 \min_{\bm{\beta}}~\frac{1}{2}\big\| H\bm{\beta}-T \big\| ^2 + \frac{\mu }{2} \big\|\bm{\beta}  \big\|^2,
\end{equation}
where the regularization term $\frac{\mu}{2} \|\cdot \|^2$ can make the resultant solution stabler and tend to have better generalization performance if the trade-off parameter $\mu $ is chosen appropriately \cite{huang2012extreme}.
$H\in\mathbb{R}^{N\times L}$ represents the hidden layer output matrix,
\begin{equation}
H =\left[\begin{matrix}
\bm{h}(\bm{X}_{1}) \\
\vdots \\
\bm{h}(\bm{X}_{N})\end{matrix}\right]= 	
\left[\begin{matrix}
h_1(\bm{X}_{1})  & ... & h_L(\bm{X}_{1}) \\
\vdots &\vdots & \vdots \\
h_1(\bm{X}_{N})  & ... & h_L(\bm{X}_{N})
\end{matrix}\right].
\end{equation} 
$T=\left[T_1,...,T_N\right]^T$ denotes the outputs of training data.
 According to \cite{huang2012extreme}, the closed-form solution for (2) can be obtained as 
\begin{equation}
\bm{\beta}^* =  \big( H^TH + \mu I  \big)^{-1}H^TT.
\end{equation} 
Thus the output of ELM for $x$ is
\begin{equation}
y_L(x)=\bm{h}\left(x\right)  \big( H^TH + \mu I  \big)^{-1}H^TT.
\end{equation}

\begin{figure}
	\vskip 0.2in
	\centering
	\subfigure[Typical ELM for task $t$.]{
		\includegraphics[width=63 mm]{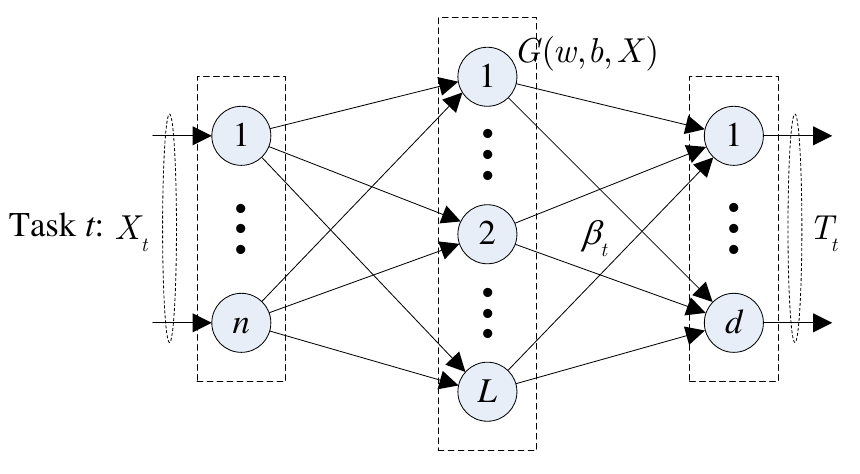}  		 
	}
	\subfigure[ELM based multitask learning machine for task $t$.]{ 
		\includegraphics[width=80 mm]{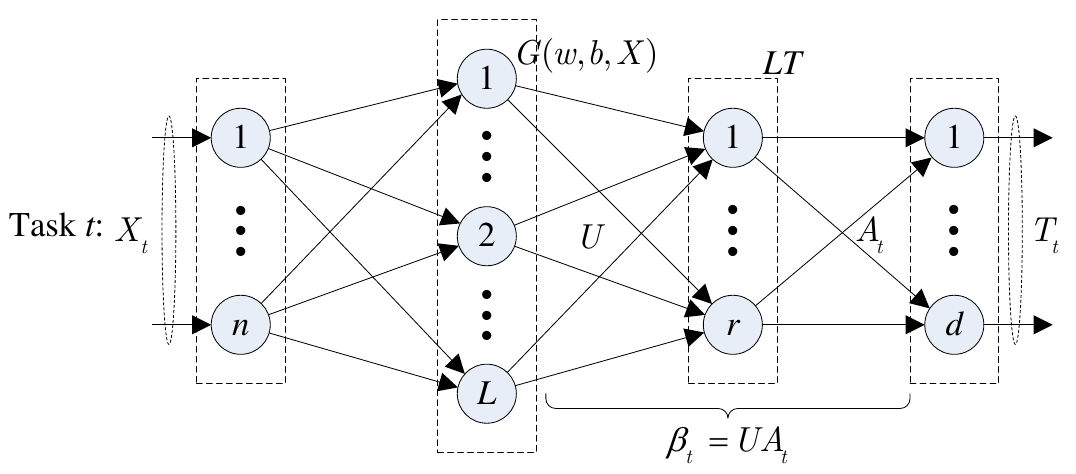}  	 
	}
	\caption{Two different model structures for task $t$: (a) typical ELM; (b) multi-task learning based on ELM with a common shared subspace $U$ and task specified weight $A_t$. }  
	\vskip -0.2in
\end{figure}

 ELM can be implemented on a single machine since the parameters are obtained explicitly by (5).

\subsection{Centralized MTL with ELM } 
 
In what follows, we apply the ELM method to MTL, where each task implements its own ELM scheme. Through learning a shared feature space of the hidden layer among tasks to transfer the knowledge of training data, we present how each task can improve the generalization performance. Suppose that there are $m$ tasks with each data set as $D_t= \{ (\bm{X}_{ti},T_{ti} )|i=1,...,N_t,t=1,...,m  \}  $, where $N_t$ is the number of training samples. Therefore the corresponding hidden layer feature space for task $t$ is expressed as $H_t = [ \bm{h}(\bm{X}_{t1})^T, ... ,\bm{h}(\bm{X}_{tN_t})^T  ]^T\in \mathbb{R}^{N_t\times L} $.  
The output of ELM for task $t$ is $f_t=H_t\bm{\beta}_t$. Following \cite{kumar2012learning}, we assume that there are $r$ latent basis tasks in MTL and each observed task can be represented as a subset of these basis tasks. Thus there
exists a subspace of the hidden layer feature space, with which all the tasks can share related representation. Hence from \cite{kumar2012learning}, the output weights can be evaluated by $\bm{\beta}_t=UA_t$, where $U\in \mathbb{R}^{L\times r}$ presents the predictive structure of the tasks while $A_t\in\mathbb{R}^{r\times d}$ determines the grouping structure. 
But different from \cite{kumar2012learning}, in the multi-task feature learning based on ELM (MTL-ELM), the input data $\bm{X}$ is first randomly mapped into a hidden feature space. Then we can conclude the structure of MTL-ELM as Fig. 1 (b), where a common \textit{Linear Transform} (LT) layer is stacked between the hidden layer and output layer for each task. $U$ is the weight for LT layer and shared among tasks. $A_t$ is the output weight and is task specific.

To obtain the optimal feature space $U$ and $\{A_t \}$, the weights for the LT layer and task-specified output layer, we will solve the following problem (6), which considers $\ell_2$-norm penalty of both $U$ and $\{A_t |t=1,...,m\}$,

\begin{equation} \label{eq:1}
\min_{U,\bm{A}}~ \sum\nolimits_{t } \frac{1}{2}  \big\| H_tUA_t-T_t  \big\| ^2 + \frac{\mu_1}{2} \big\| U\big\|^2  +  \frac{\mu_2}{2} \big\|\bm{A} \big\|^2 ,
\end{equation}
where $\bm{A}= [A_1^T,...,A_m^T ]^T$ and $T_t \in  \mathbb{R} ^{N_t\times d}$. $\mu_1$ and $\mu_2$ are the trade-off parameters between the training error and regularization. It is assumed that each ELM utilizes the same number of hidden nodes $L$. The random weights and biases of hidden neurons are the same as $\left\{\bm{w}_l,b_l|l=1,...,L\right\}$.

The cost function in (6) is convex in $U$ for fixed $\bm{A}$, and is convex in $A_t$ for a fixed $U$. However it is not jointly convex.   
Hence we adopt \textit{alternating optimization} (AO) method to achieve a local minimum for problem (6). When $\bm{A}^k$ is fixed, for the iteration $k+1$, the problem reduces to 
\begin{equation}
U^{k+1}:=\arg\min_{U} \sum\nolimits_{t} \frac{1}{2} \big\|H_tUA_t^k-T_t \big\|^2+\frac{\mu_1}{2}  \big\|U \big\|^2.
\end{equation}
(7) is convex in $U$ and a closed-form solution can be obtained for MTL-ELM.
Setting the derivative to zero we can obtain
\begin{equation}
\sum\nolimits_{t } H_t^TH_tU^{k+1}A_t^kA_t^{kT} + \mu_1 U^{k+1} = \sum\nolimits_{t } H_t^TT_tA_t^T  ,
\end{equation}
 where we refer to ${ (X^k )}^T$ as $X^{kT}$ for simplicity. We apply vectorization on both sides of equation (8). According to the property of Kronecker product $\text{vec}(AXB)= (B^T\otimes A )\text{vec}(X)$, we have
\begin{equation}
\begin{aligned}
\text{vec}\big(U^{k+1}\big) =& \big( \sum\nolimits_{t } \big(A_t^k A_t^{kT}\big) \otimes \big(H_t^T H_t\big)  + \mu_1 I \big)^{-1} \sum\nolimits_{t } \text{vec}  \big(H_t^TT_tA_t^{kT} \big).
\end{aligned}
\end{equation}

\begin{algorithm} [t]
	\SetAlgoLined
	\KwIn{Hidden layer feature space $\{H_t|t=1,..., m\}$}
	\KwOut{$U$ and $\{ A_t|t=1,..., m\}$ }
	initialization: $ \{A_t^0=\bm{1} |t=1,...,m \}$\;
	\For{$k=0,1,...$ }{					 
		Update $U^{k+1}$ according to (9)\;
		\For{$t=1$ {\bfseries to} $m$}{
			Update $A_t^{k+1}$ according to (11)\;
		}
	}	
	\caption{MTL-ELM}
\end{algorithm}

When $U^{k+1}$ is fixed, the problem reduces to solving (10) for each task separately,
\begin{equation}
A_t^{k+1}:=\arg\min_{A_t } \frac{1}{2} \big\|H_tU^{k+1}A_t-T_t \big\|^2 + \frac{\mu_2}{2}  \big\|A_t \big\| ^2.
\end{equation}
(10) is convex over $A_t$ and the solution can be easily obtained from \cite{huang2006extreme}, where
\begin{equation}
A_t^{k+1} =  \big( U^{(k+1)T}H_t^TH_tU^{k+1} + \mu_2 I  \big)^{-1} U^{(k+1)T}H_t^TT_t   .
\end{equation}
Then we propose the multi-task learning based on ELM algorithm (MTL-ELM) to solve (6). In MTL-ELM, each task initializes $A^0_t=\bm{1} $. Then the iterations are executed, and the shared representation $U$ and $\{A_t\}$ are successively optimized in each iteration. Based on the analysis in \cite{beck2015convergence}, we provide the convergence of MTL-ELM as follow.

 \begin{lemma}
 	Let $\{U^k,\bm{A}^k \}$ be the sequence generated by MTL-ELM. Then with $k\to \infty$, the sequence $\{U^k,\bm{A}^k \}$ converges to a stationary point $(U^*,\bm{A}^*)$ of problem (6).
 \end{lemma}
 \begin{proof}
 See Appendix A.
 \end{proof}

\section{Decentralized multi-task learning with ELM }
\subsection{Motivation and basics}
In many real-world applications, the data of different tasks may be geo-distributed over different machines. However, the MTL-ELM algorithm proposed in different locations can only work on a single machine, which has the whole data sets of all tasks. Due to the heavy communication load caused by transmitting the data or the constraint of data privacy and security, it may be impossible to send data of different tasks to a master machine to perform the MTL-ELM algorithm. Therefore we consider the decentralized multi-task feature learning based on ELM (DMTL-ELM), where the shared representation $U$ for hidden layer feature is obtained through information exchange between tasks. Meanwhile, different from the distributed ELM (DELM) in \cite{luo2017distributed} where one identical output is desired, we learn the distinctive predictive model for each task based on its local data distribution.  

In DMTL-ELM, we consider a setup with $m$ tasks distributed over $m$ agents, i.e., one agent for one task. We denote the multi-agent network as an undirected graph $\mathcal{G}=\{\mathcal{V},\mathcal{E}\}$, where $\mathcal{V}=\{1,...,m\}$ is the agents and $\mathcal{E}$ includes the connections. 
\begin{assumption}
	The undirected graph $\mathcal{G}$ is connected.
\end{assumption}
\noindent
Assumption 1 implies that any two agents in the network can always influence each other in the long run.
Each agent implements its own ELM scheme locally with the same number of hidden neurons as $L$ and the identical random weights and biases as $\{\bm{w}_l,b_l|l=1,...,L \}$. The data set for task $t$ is $D_t (t=1,...,m )$, which is the same as that in centralized MTL. We assume that the data sets $\{D_t\}$ and hidden feature space $\{H_t\}$ of each task cannot be shared due to costly communication load or security consideration, while other information transmission among agents is allowed. 
 
The goal of DMTL-ELM is to find the sharing representation $U$ cooperatively across agents, as well as the predictive model $\{A_t|t=1,...,m\}$ for each agent by minimizing the overall expected loss. 
One naive approach to solving (6) in a decentralized way is to exchange the information $ (A_t^kA_t^{kT} )\otimes  (H_t^TH_t )$ and $H_t^TT_tA_t^{kT}$ among agents. However, when dimensions $L$ and $r$ are large, the amount of exchanged information becomes huge. This significantly increases communication load. Another limitation is that this strategy can only work efficiently in the master-slave structure, where the master gathers all the exchanged information.
Hence we will solve problem (6) with decentralized optimization method instead, where an identical and optimal subspace $U$ is ensured among agents. In the following, we will consider the decentralized network structure and propose the corresponding multi-task learning algorithm.

\subsection{Problem Formulation} 

Since each agent can only access to its own data set in the decentralized learning setup, the global optimal solution for $U$ cannot be obtained directly if data sharing is not allowed across agents.   
In order to find optimal $U$, information on local weights needs to be exchanged among agents. Thus, we exploit the alternating direction method of multipliers (ADMM) to ensure all machines agree with the same subspace by solving the following problem,    

\begin{equation}
\begin{aligned}
\min_{\bm{U},\bm{A}} &~\sum\nolimits_{t}  \big(\frac{1}{2}  \big\| H_tU_tA_t-T_t  \big\| ^2 + \frac{\mu_1}{2m} \big\|U_t \big\|^2 + \frac{\mu_2}{2} \big\|A_t \big\|^2  \big),\\
\vphantom{\frac{1}{2}}s.t.  &~ \sum\nolimits_{t} C_tU_t=\bm{0}. 
\end{aligned}
\end{equation}
Compared with the centralized scenario, agent $t$ holds a local subspace variable $U_t$. To ensure a unique subspace, the constraint of $\left\{U_t \right\}$ is introduced. For any $(i,j)\in \mathcal{E} (i,j\in\mathcal{V})$, the subspaces of $U_i=U_j$ is required. The $C_t\in\mathbbm{R}^{ |\mathcal{E}| L\times L}$ can be deduced from $\mathcal{G}$. 
 
According to ADMM \cite{wang2018distributed}, the augmented Lagrangian function for problem (16) is given by
\begin{equation}
\begin{aligned}
	\vphantom{\frac{1}{2}}\mathcal{L}   \big(   \bm{U },  \bm{A} ,\lambda   \big)&= \sum\nolimits_{t }  \big(\frac{1}{2} \big\|H_tU_tA_t-T_t \big\|^2+\frac{\mu_1}{2m} \big\|U_t \big\|^2 +  \\
&	\vphantom{\frac{1}{2}} \frac{\mu_2}{2 } \big\|A_t \big\|^2  \big)  + \lambda^T\sum\nolimits_{t } C_tU_t  + \frac{\rho}{2} \big\| \sum\nolimits_{t } C_tU_t   \big\|^2,
\end{aligned}
\end{equation}
where $\lambda$ is a Lagrange multiplier and $\bm{U}= [U_1^T,...,U_m^T  ]^T$, $\rho>0$. Since Lagrangian $\mathcal{L}$ is separable in $U_t$ and $A_t$, we can optimize them in parallel across the agents. However, since $U_t$ and $A_t$ are coupled variables, we can follow AO to update them sequentially. Hence from the iteration of proximal Jacobi ADMM \cite{Deng2017}, we update the variables by
 
\begin{algorithm} [t]
	\SetAlgoLined
	\KwIn{ $\{H_t,T_t,C_t,P_t,Q_t,|t=1,..., m\}$, $\delta$} 
	\KwOut{$\{U_t,A_t|t=1,..., m\}$} 
	Initialization: $ \{U^0_t=\bm{1},A_t^0=\bm{1} |t=1,...,m \} $, $\lambda^0=\bm{0}$\;

	\For{$k=0,1,...$   }{ 
		
		{\bfseries agents $t=1$ to $m$}:\\
		Update $U_t^{k+1}$ \textit{in parallel} by (19), and share with the neighbouring agents\;  
		Pick $\gamma_i^{k+1}\in\big(0,\frac{\delta \| \hat{C}_i(\bm{U}^{k }- \bm{U}^{k+1})  \|^2 }{  \| \hat{C}_i\bm{U}^{k+1}  \|^2 }\big], i=1,...,|\mathcal{E}|$\;
		Update $\lambda^{k+1}:=\lambda^k -\rho \bm{\gamma}^{k+1}  \sum_{t} C_tU_t^{k+1}$.\\
		Update $A_t^{k+1}$ \textit{in parallel} by (21).
		
	}
	
	\caption{DMTL-ELM}
\end{algorithm}

\begin{align}
U_t^{k+1}:=&\arg \min_{U_t}~\mathcal{L}  \big( U_t,\bm{U}_{-t}^k,A_t^k,\lambda^k  \big)+\frac{1}{2}  \big\|U_t-U_t^k  \big\|^2_{P_t},  \\
A_t^{k+1}:=&\arg \min_{A_t}~\mathcal{L}  \big( U_t^{k+1},A_t,\lambda^k  \big) + \frac{1}{2}  \big\|A_t-A_t^k   \big\|^2_{Q_t}, \\
\vphantom{\frac{1}{2}}\lambda^{k+1}:=&\lambda^k-\rho\bm{ \gamma}^{k+1} \sum\nolimits_{t}C_tU_t^{k+1}, 
\end{align} 
where $\bm{U}_{-t}\triangleq \{U_1,...,U_{t-1},U_{t+1},...,U_m \}$. Since $\{U_t\}$ and $\{A_t\}$ are matrices, for $z\in\mathbbm{R}^{n\times m }$, we define the \textit{G-norm} as $ \|z\|^2_G=\text{tr}(z^TGz)$. Without loss of generality, we assume $\{P_t\}$ and $\{Q_t\}$ are diagonal matrices of which the diagonals are positive.
In the updated process of $U_t^{k+1}$, it is equivalent to solve
\begin{equation}
\begin{aligned}
U_t^{k+1}:=\arg \min_{U_t} & \frac{1}{2}  \big\|H_tU_tA_t^k-T_t  \big\|^2 +\frac{\mu_1}{2m} \big\|U_t \big\|^2 + \lambda^{kT}C_tU_t+\\
&\frac{\rho}{2}   \big\|C_tU_t+\sum\nolimits_{i\neq t}C_iU_i^k  \big\|^2+\frac{1}{2}  \big\|U_t-U_t^k    \big\|^2_{P_t}.
\end{aligned}	
\end{equation}
Setting the derivative of (17) to $0$ leads to 
\begin{equation}
\begin{aligned}
H_t^TH_t&U_t^{k+1}A^k_t {A_t}^{kT} +  \big( \frac{\mu_1}{m}I + \rho C_t^TC_t + P_t \big)U_t^{k+1}= \\
	\vphantom{\frac{1}{2}}&H_t^TT_t{A_t}^{kT} -\rho C_t^T  \sum\nolimits_{i\neq t}C_iU_i^k- C_t^T \lambda^k  + P_tU_t^k.
\end{aligned} 
\end{equation}
Applying vectorization on both sides of (19), we obtain
\begin{equation}
\begin{aligned}
	\vphantom{\frac{1}{2}}\text{vec}  \big(U_t^{k+1} \big) &=  \big( \big(A_t^kA_t^{kT}  \big)\otimes   \big(H_t^TH_t  \big)+ I\otimes   \big( \frac{\mu_1}{m}I + \rho C_t^TC_t + P_t    \big)   \big)^{-1}\\
 	\vphantom{\frac{1}{2}}& \text{vec} \big( H_t^TT_t{A_t}^{kT}- \rho C_t^T \sum\nolimits_{i\neq t}C_iU_i^k - C_t^T\lambda^k + P_tU_t^k  \big).
\end{aligned}
\end{equation} 
While updating $A_t^{k+1}$ as (10), it is equivalent to solve
\begin{equation}
\begin{aligned}
A_t^{k+1}:=\arg\min_{A_t} \frac{1}{2} \big\|  H_t&U_t^{k+1}A_t-T_t \big\|^2+ \\& \frac{\mu_2}{2 }\big\|A_t\big\|^2 + \frac{1}{2}  \big\| A_t-A_t^k  \big\|_{Q_t}^2 .
\end{aligned}
\end{equation}
Setting the derivative to $0$, we get the optimal solution of (21) as
\begin{equation} 
\begin{aligned}
	\vphantom{\frac{1}{2}}A_t^{k+1}= & \big(U_t^{(k+1)T}H_t^TH_tU_t^{k+1}+Q_t+ \mu_2 I\big)^{-1}\cdot\\	\vphantom{\frac{1}{2}}&\big(U_t^{(k+1)T}H_t^TT_t+Q_tA_t^k  \big).
\end{aligned}
\end{equation}

With above analysis, we summarize the DMTL-ELM algorithm in Algorithm 2, where $\bm{\gamma}=\text{blkdiag}(\gamma_1I,...,\gamma_{|\mathcal{E}|}I )\in\mathbbm{N}^{|\mathcal{E}|L\times|\mathcal{E}|L}$, $\bm{C}=[C_1,\cdots,C_m]$ and $\bm{C}=[\hat{C}_1^T,...,\hat{C}_{|\mathcal{E}|}^T]^T$ with $\hat{C}_i\in\mathbbm{R}^{L\times mL}$. $\delta(>0)$ is predetermined parameter and $\bm{\gamma}$ can balance the primal and dual residuals, which will be given by following analysis. In the $k$-th iteration of Algorithm, $U_t$ and $A_t$ are updated successively at each agent. Moreover, the update process of $U_t$ and $A_t$ is carried out in parallel across agents. Hence the DMTL-ELM approach is a hybrid Jacobian and Gauss-Seidel Proximal multi-block ADMM algorithm. Since a unique subspace is expected for all agents, the update of $U_t^{k+1}$ requires the information of $U_t^{k}$ and $ \{U_i^{k}|i\in\mathcal{V}_i  \}$, where $\mathcal{V}_i$ includes the agents that connect with agent $t$. $\{A_t\}$, as the local predictive model, are specific among agents, and are updated privately without any sharing.

\subsection{Computation efficient DMTL-ELM}
Regarding the proposed algorithm DMTL-ELM, the update of $U_t$ and $A_t$ will consume substantial computing resource since the matrix inverse needed to be calculated in each iteration. This becomes especially severe when the dimension of hidden feature $L$ and the number of basic tasks $r$ are large. To reduce the burden on the computation of DMTL-ELM, we consider applying a first-order approximation for the updating process of $U_t$. Denoting $F_t\left(U_t,A_t\right) =\frac{1}{2} \|H_tU_tA_t-T_t \|^2 + \frac{\mu_1}{2m}  \|U_t  \|^2+\frac{\mu_2}{2 }  \|A_t  \|^2$, the function is approximated by $F_t(U_t,A_t^k)\approx F_t(U_t^k,A_t^k) + \langle \nabla_{U_t}F_t (U_t^k,A_t^k ),U_t-U_t^k  \rangle $ with fixed $A_t^k$. Then the update of $U_t$ for agent $t$ can be evaluated as
\begin{equation}
\begin{aligned}
 \vphantom{\frac{1}{2}}U_t^{k+1}:=\arg\min_{U_t}~&\big\langle \nabla_{U_t}F_t \big(U_t^k,A_t^k \big),U_t-U_t^k  \big\rangle+\lambda^{kT}C_tU_t+\\
& \frac{\rho}{2} \big\| C_tU_t+ \sum\nolimits_{i\neq t}  C_iU_i^k\big\|^2 + \frac{1}{2}  \big\|U_t-U_t^k \big\|^2_{P_t}.\\
\end{aligned}
\end{equation} 
Hence the closed form representations of $U_t^{k+1}$ is given by
\begin{equation}
	\begin{aligned}
		\vphantom{\frac{1}{2}}U_t^{k+1}=& \big(\rho C_t^TC_t + P_t  \big)^{-1} \big( -H_t^TH_tU_t^{k } A^k_t {A_t}^{kT} + H_t^TT_t{A_t}^{kT} - \\
	&	\vphantom{\frac{1}{2}}\frac{\mu_1}{m}U_t^k-\rho C_t^T  \sum\nolimits_{i\neq t}C_iU_i^k- C_t^T \lambda^k  + P_tU_t^k \big). 
	\end{aligned}
\end{equation}
 
%

Comparing (23) with (19), the calculation of inversion becomes fixed, which is only associated with the connection constraint $C_t$ and penalties $\rho,P_t$. Hence the complexity of computation reduces. Moreover, if we further introduce the approximation for the update of $A_t$, (21) can be reduced to the gradient descent method. By substituting (19) with (23) in algorithm 2, we can obtain the Algorithm 3, namely first-order DMTL-ELM (FO-DMTL-ELM).
 
\subsection{Convergence Analysis}
In what follows, we will analyze the convergence properties of the proposed DMTL-ELM and FO-DMTL-ELM algorithms. Denoting $F(\bm{U},\bm{A})=\sum_{t}F_t (U_t,A_t )=\sum_{t}(f_t (U_t,A_t )+g_1 (U_t )+g_2 (A_t ))$, where $f_t (U_t,A_t )=\frac{1}{2}  \| H_tU_tA_t-T_t   \|^2$, $g_1(\cdot)=\frac{\mu_1}{2m}\|\cdot\|^2$ and $g_2(\cdot) =\frac{\mu_2}{2 }\|\cdot\|^2$, then problem (12) can be generalized as
\begin{equation}
\begin{aligned}
\vphantom{\frac{1}{2}}\min_{\bm{U},\bm{A} } &\sum\nolimits_{t} (f_t (U_t,A_t )+g_1 (U_t )+g_2 (A_t ) ), ~
s.t.~ \bm{C}\bm{U}+\bm{D}\bm{A}=\bm{b},
\end{aligned}
\end{equation} 
where $\bm{D}=\bm{0}$ and $\bm{b}=\bm{0}$.
It is worth noting that $F_t(U_t,A_t)$ can only be convex when either $U_t$ or $A_t$ is fixed. Hence (24) is a majorized multi-block problem with coupled bi-convex objective functions. A more generalized presentation of (24) is the majorized model, which tries to solve  
\begin{equation}
\begin{aligned}
\vphantom{\frac{1}{2}}\min_{\bm{U},\bm{A} } \phi \left(\bm{U} ,\bm{A} \right) +  \sum\nolimits_{t} \varphi_t (U_t ) + \sum\nolimits_{t}\psi_t  (A_t  ), ~s.t.~\bm{C}\bm{U}+\bm{D}\bm{A}=\bm{b}. 
\end{aligned}	
\end{equation}
From the Multi-block ADMM algorithm in \cite{Gao2017}, problem (25) can be solved by the BCD method with Gauss-Seidel type, which optimizes the variable sequentially while fixing the remaining blocks at their last updated values. However, the sequential update behavior is non-efficient for solving the decentralized optimization problem. 
In \cite{xu2018hybrid}, a hybrid Jacobian and Gauss-Seidel proximal block coordinate update (BCU) method is presented to solve a linearly constrained multi-block structured problem with a quadratic term in the objective. In \cite{Banjac2018} it shows that problem (25) can also be solved by a regularized version of the Jacobi algorithm. However, all the references mentioned above assume that the objective function is convex, which hence cannot be applied to solving (24). Though the methods for multi-convex and nonconvex optimization are provided in \cite{Xu2013} and \cite{Wang2019}, they only focus on studying the algorithms with Gauss-Seidel type.

The fully Gauss-Seidel update usually performs better than the fully Jacobian update empirically \cite{xu2018hybrid}, we integrate Jacobi-Proximal ADMM \cite{Deng2017} with Gauss-Seidel update in Algorithm 2 to solve problem (24) in a hybrid way, which is presented as the DMTL-ELM. 
Because $f_t(\cdot,\cdot)$ is bi-convex and $g_1 (\cdot)$ and $g_2 (\cdot)$ are strongly convex functions, following inequalities are useful for proving the convergence of DMTL-ELM and FO-DMTL-ELM, where $ \langle A,B  \rangle = \text{tr} (A^TB  ) (A,B\in\mathbbm{R}^{n\times m})$.
\begin{proposition} For any $ U^1_t,U^2_t \in \bm{\mathcal{U}}$ and $A_t^1,A_t^2\in \bm{\mathcal{A}}$, we have
	\begin{equation}\label{eq:2}
	\begin{aligned}
		&\vphantom{\frac{1}{2}}	f_t  \big(U_t^1,A_t^1 \big)-f_t  \big(U_t^2,A_t^1  \big) \geq  \big\langle  \nabla f_t \big(U_t^2,A_t^1\big), U_t^1-U_t^2     \big\rangle ,\\
		&\vphantom{\frac{1}{2}}	f_t  \big(U_t^1,A_t^1 \big)-f_t  \big(U_t^1 ,A_t^2  \big) \geq  \big\langle  \nabla f_t \big(U_t^1 ,A_t^2\big), A_t^1-A_t^2     \big\rangle .
	\end{aligned}
	\end{equation}
With the strong convexity of $g_1 (\cdot )$ and $g_2 (\cdot)$, we have
\begin{equation}\label{eq:4}
\begin{aligned}
&\vphantom{\frac{1}{2}} g_1\big(U_t^1\big) - g_1\big(U_t^2\big)\geq \big \langle g'_1\big(U_t^2\big), U_t^1-U_t^2 \big\rangle +\frac{\sigma}{2}  \big\| U_t^1-U_t^2  \big\|^2  , \\
& \vphantom{\frac{1}{2}} g_2\big(A_t^1\big) - g_2\big(A_t^2\big)\geq \big\langle g'_2\big(A_t^2\big), A_t^1-A_t^2 \big\rangle +\frac{\sigma}{2}  \big\| A_t^1-A_t^2  \big\|^2.
\end{aligned}
\end{equation}  
\end{proposition}
\begin{proof}	
   From the bi-convexity of function $f_t(\cdot,\cdot)$, (26) can be shown directly. 
   Since $\sigma\in(0,2)$ can make (27) satisfied, $g_1(\cdot)$ and $g_2(\cdot)$ are strongly convex.
\end{proof}

Since $\rho C_t^TC_t$ is diagonal matrix, we follow the \textit{prox-linear case} to set diagonal matrices $P_t = \tau_t I-\rho C_t^TC_t$ and $Q_t = \zeta_t I$. To prove the convergence of (FO-)DMTL-ELM, Our analysis mainly focuses on showing that the augmented Lagrangian $\mathcal{L}$ defined in (13) is lower bounded and monotonically non-increasing with iterations. Defining $\sigma_{t,\text{max}}$ as the largest eigenvalue of $C_t^TC_t$, then we have the following result.
\begin{lemma}(Sufficient descent of $\mathcal{L}$ during $\bm{U}$ and $\lambda$ update)
	For a given $\delta>0$, choosing $\gamma_i^{k+1}\in\big(0, \frac{\delta \| \hat{C}_i(\bm{U}^k- \bm{U}^{k+1} )  \|^2 }{  \| \hat{C}_i\bm{U}^{k+1}  \|^2 }\big](i=1,...,|\mathcal{E}|)$ and $\tau_t\geq \rho m (\delta+\frac{1}{2})\sigma_{t,\max}-\frac{\sigma}{2}$, the iterations in DMTL-ELM satisfy
	\begin{equation}
		\vphantom{\frac{1}{2}}\mathcal{L}\big(\bm{U}^k,\bm{A}^k,\lambda^k\big)-\mathcal{L}\big(\bm{U}^{k+1},\bm{A}^{k},\lambda^{k+1}\big)\geq 0.
	\end{equation} 
\end{lemma}
\begin{proof}
	See the Appendix B.
\end{proof}
Since the agents are fully decentralized, a unique $\gamma^{k+1}$ can not be guaranteed. Supposing  $C_t^TC_t=|d_t|I$ and thus $ \|\hat{C}_i\bm{U}^{k+1} \|^2=  \|U_t^{k+1}-U_j^{k+1} \|^2$ where $(t,j)\in\mathcal{E}$, it demonstrates that $\gamma_i^{k+1}$ relates to the \textit{primal} residual among agents $t$ and $j$. The main idea is to bound the \textit{primal} residue by the \textit{dual} residual. This is because $ \|U_t^{k+1}-U_j^{k+1} \|^2\leq \frac{\delta}{\gamma_i^{k+1}} \|(U_t^k-U_j^k) - (U_t^{k+1}-U_j^{k+1})  \|^2\leq \frac{2\delta}{\gamma_i^{k+1}}( \| U_t^k-U_t^{k+1} \|^2 + \| U_j^{k}-U_j^{k+1}  \|^2)$. Hence the $\delta$ and $\gamma_i^{k+1}$
can balance the \textit{primal} and \textit{dual} residuals. Further, $\delta$ can also be adjusted in each iteration.
\begin{lemma}(Sufficient descent of $\mathcal{L}$ during $\bm{A}$ update) Letting $\zeta_t\geq 0$, then the iterations in DMTL-ELM satisfy
	\begin{equation}
		\vphantom{\frac{1}{2}}\mathcal{L}\big(\bm{U}^{k+1},\bm{A}^k,\lambda^{k+1}\big)-\mathcal{L}\big(\bm{U}^{k+1},\bm{A}^{k+1},\lambda^{k+1}\big)\geq 0.
	\end{equation}	
\end{lemma} 
 \begin{proof}
 	See the Appendix C.
 \end{proof}
To combine the Lemma 2 and Lemma 3, we can conclude that 
\begin{equation}
\begin{aligned}
	\vphantom{\frac{1}{2}}\mathcal{L}\big(\bm{U}^{k },&\bm{A}^k,\lambda^{k }\big)-\mathcal{L}\big(\bm{U}^{k+1},\bm{A}^{k+1},\lambda^{k+1}\big)\\
	&\vphantom{\frac{1}{2}}\geq \sum\nolimits_{t}\big(c_{t,1} \big\|U_t^k-U_t^{k+1} \big\|^2 + c_{t,2}\big\|A_t^k-A_t^{k+1} \big\| ^2  \big),	
\end{aligned}
\end{equation} 
where $c_{t,1}=\tau_t + \frac{\sigma}{2}-\rho m (\delta+\frac{1}{2})\sigma_{t,\max} $ and $c_{t,2} =\zeta_t+\frac{\sigma}{2}$. Moreover, when $c_{t,1},c_{t,2}\geq 0$, the decrease of $\mathcal{L}$ with iterations in DMTL-ELM can be guaranteed. 

\begin{lemma}(Boundedness)
	The sequence $\{\bm{U }^k,\bm{A}^k,\lambda^k \}$ generated by DMTL-ELM is bounded. $\mathcal{L}(\bm{U }^k,\bm{A}^k,\lambda^k)$ is lower bounded for all $k\in\mathbbm{N}$ and converges as $k\to \infty$.
\end{lemma}
\begin{proof}
	See the Appendix D.
\end{proof}

\begin{lemma}(Partial gradient bound) For any partial gradient $\partial \mathcal{L}(\bm{U}^{k+1},\bm{A}^{k+1},\lambda^{k+1})$, there exists $c>0$ such that
	\begin{equation}
	 \vphantom{\frac{1}{2}} \partial \mathcal{L}\big(\bm{U}^{k+1},\bm{A}^{k+1},\lambda^{k+1}\big)\leq c\sum\nolimits_{t}\big(\big\|U_t^k-U_t^{k+1} \big\|^2 + \big\|A_t^k-A_t^{k+1} \big\|^2\big).
	\end{equation}
\end{lemma}
\begin{proof}
	See the Appendix E.
\end{proof}
 
Based on Lemmas 2,3 and 4, we now establish the convergence of DMTL-ELM.
 
\begin{theorem}
	Letting $\tau_t\geq \rho m (\delta+\frac{1}{2})\sigma_{t,\max}-\frac{\sigma}{2}$ and $\zeta_t\geq0$, the sequence $\{\bm{U}^k,\bm{A}^k,\lambda^k\}$ generated by DMTL-ELM converges to a stationary point $(\bm{U}^*,\bm{A}^*,\lambda^*)$ of $\mathcal{L}$ as $k\to \infty$.

\end{theorem}
\begin{proof}
 By Lemma 2 and Lemma 3, $\mathcal{L}(\bm{U}^k,\bm{A}^k,\lambda^k)$ is monotonically non-increasing and lower bounded, and therefore from (31), $ \|U_t^k-U_t^{k+1}  \|^2\to 0$ and $ \|A_t^k-A_t^{k+1}  \|^2 \to 0$ as $k\to \infty$. Moreover from the proof of Lemma 2, $ \|\lambda^k-\lambda^{k+1}  \|^2\to 0$ also holds when $k\to\infty$. Supposing that the sequence $(\bm{U}^k,\bm{A}^k,\lambda^k)$ converges to a limit point $(\bm{U}^*,\bm{A}^*,\lambda^*)$, then from Lemma 4, there exists a convergent subsequence $(\bm{U}^{k_s},\bm{A}^{k_s},\lambda^{k_s})_{s\in\mathbbm{N}}$ such that $(\bm{U}^{k_s},\bm{A}^{k_s},\lambda^{k_s})\to (\bm{U}^*,\bm{A}^*,\lambda^*)$ as $s \to \infty$. Based on Lemma 5, $ \|\partial\mathcal{L}(\bm{U}^k,\bm{A}^k,\lambda^k)  \|^2\to 0$ when $k\to\infty$, and in particular $ \|\partial\mathcal{L}(\bm{U}^{k_s},\bm{A}^{k_s} ,\lambda^{k_s})  \|^2\to 0$ as $s\to \infty$. Since $f_t(\cdot,\cdot)$, $g_1(\cdot)$ and $g_2(\cdot)$ are continuous, we have $\lim\limits_{s\to\infty}\mathcal{L}(\bm{U}^{k_s},\bm{A}^{k_s},\lambda^{k_s})\to\mathcal{L}(\bm{U}^*,\bm{A}^*,\lambda^*)$. By Proposition 2 \cite{Wang2019}, we have $\partial \mathcal{L}(\bm{U}^*,\bm{A}^*,\lambda^*)=0$, which demonstrates that the limit point $(\bm{U}^*,\bm{A}^*,\lambda^*)$ is a stationary point. That concludes the proof.
\end{proof}

Following the analysis for DMTL-ELM, we next provide the convergence analysis of FO-DMTL-ELM. 
\begin{proposition} (Block-coordinate Lipschitz continuous) For $U_t^1,U_t^2\in\mathcal{U}$ and $A_t\in\mathcal{A}$, there exists $L_t>0$\footnote[1]{Since the boundedness of $\bm{U}$ and $\bm{A}$ is proved in Lemma 3, from the proof of Lemma 5 we suppose that the Lipschitz constant for gradient $\nabla F(\bm{U},\bm{A})$ in bounded set is consistent with the block-coordinate Lipschitz continuous constant $L_t$.} such that
	\begin{equation}
	\vphantom{\frac{1}{2}}	\big\|\nabla_{U_t}F_t\big(U_t^1,A_t\big)- \nabla_{U_t}F_t\big(U_t^2,A_t\big) \big\| \leq L_t \big\| U_t^1 - U_t^2  \big\|.
	\end{equation} 
\end{proposition}
\begin{proof}
  Refer to the proof of Lemma 1, but with a different $L_t$ rather than $L_U$.
\end{proof}

\begin{lemma}(Sufficient descent of $\mathcal{L}$ during $\bm{U}$ and $\lambda$ update)
Following the same update strategy for $\gamma_i^{k+1}$ in Lemma 2, by guaranteeing $\tau_t\geq L_t+ \rho m (\delta+\frac{1}{2})\sigma_{t,\max}-\frac{\sigma}{2}$, then the iterations in FO-DMTL-ELM satisfy
\begin{equation}
\vphantom{\frac{1}{2}}\mathcal{L}\big(\bm{U}^k,\bm{A}^k,\lambda^k\big)-\mathcal{L}\big(\bm{U}^{k+1},\bm{A}^{k},\lambda^{k+1}\big)\geq 0.
\end{equation}  
 \end{lemma}
\begin{proof}
		See the Appendix F.
\end{proof}
Since the update of $A_t$ is less complex compared with that of $U_t$, we do not introduce the first-order approximation. Hence with Lemma 3, the decreasing of $\mathcal{L}$ during $\bm{A}$ updating also holds with setting $\zeta_t\geq 0$. 

\begin{lemma}(Partial gradient bound) For the sequence $\{\bm{U}^k,\bm{A}^k,\lambda^k \}$ generated in FO-DMTL-ELM, and any partial gradient $\partial \mathcal{L}(\bm{U}^{k+1},\bm{A}^{k+1},\lambda^{k+1})$, there exists $c>0$ such that
	\begin{equation}
	\vphantom{\frac{1}{2}}\partial \mathcal{L}\big(\bm{U}^{k+1},\bm{A}^{k+1},\lambda^{k+1}\big)\leq c\sum\nolimits_{t}\big(\big\|U_t^k-U_t^{k+1} \big\|^2 + \big\|A_t^k-A_t^{k+1} \big\|^2\big).
	\end{equation}
\end{lemma}
\begin{proof}
	See the Appendix G.
\end{proof}

Then based on Lemmas 6 and 7, we can conclude that the convergence of FO-DMTL-ELM as follows.
\begin{theorem}
	Letting $\tau_t\geq L_t+ \rho m (\delta+\frac{1}{2})\sigma_{t,\max}-\frac{\sigma}{2}$ and $\zeta_t\geq0$, the sequence $\{\bm{U}^k,\bm{A}^k,\lambda^k\}$ generated by FO-DMTL-ELM converges to a stationary point $(\bm{U}^*,\bm{A}^*,\lambda^*)$ of $\mathcal{L}$ as $k\to \infty$.
\end{theorem}
\begin{proof}
The proof is similar as that of Theorem 1.
\end{proof}
 Compared with the results in Theorem 1, it is easy to find that the conditions required for $\tau_t$ has been modified due to the first-order approximation in the update process of $U_t$. As presented in Theorem 2, we need to choose larger proximal penalty $\tau_t$ for FO-DMTL-ELM than that of DMTL-ELM, while the conditions for other parameters are the same.

Note that the Theorem 1 and Theorem 2  present the sufficient conditions for the convergence guarantee. 
Furthermore, convergence analysis for both DMTL-ELM and FO-DMTL-ELM can be applied to the general problem (25) with multi-convex objective $\phi(\bm{U},\bm{A})$.
Since we do not specify the mapping $G(\bm{w},b,\bm{X})$, the results in Theorems 1 and 2 hold for different choice of activation functions but with appropriate parameters correspondingly. 

\section{Numerical Experiments} 
In this section, we will first provide detailed numerical results on the convergence behavior of algorithms MTL-ELM and DMTL-ELM. Then the generalization performance of proposed methods are compared with state-of-the-art approaches. Finally we will discuss the communication requirement of the decentralized method DMTL-ELM over learning accuracy. 
\subsection{Convergence experiments}

We first evaluate the convergence behavior of proposed learning approaches by solving problems (6) and (12). For simplicity, we set $m=5$, $L=\{5,10 \}$, $N_t=\{10,100\}$, $r=2$ and $d=1$. The regularization parameters are chosen as $\mu=\nu=2$. The hidden layer feature matrix $H_t$ and the training labels $T_t$ are generated randomly according to the uniform distribution $\mathcal{U}(0,1)$. Denoting $\bm{H}=[H_1^T,\cdots,H_m^T]^T$, the columns of $\bm{H}$ are normalized. For the decentralized scenario, we consider the network structure illustrated in Fig. 2 (a). We set the parameters $\gamma=1$ and $\rho=1$. With \textit{Prox-linear Proximal} we denote $P_t=\tau_tI-\rho C_t^TC_t$ and $Q_t=\zeta_t I$. Without loss of generality, we set $C_t^TC_t = d_tI$ where $d_t$ is the degree of agent $t$ in graph $\mathcal{G}$.

  \begin{figure} [b] 
 	\vskip 0.2in
 	\begin{center}
 		\centerline{\includegraphics[width=75mm]{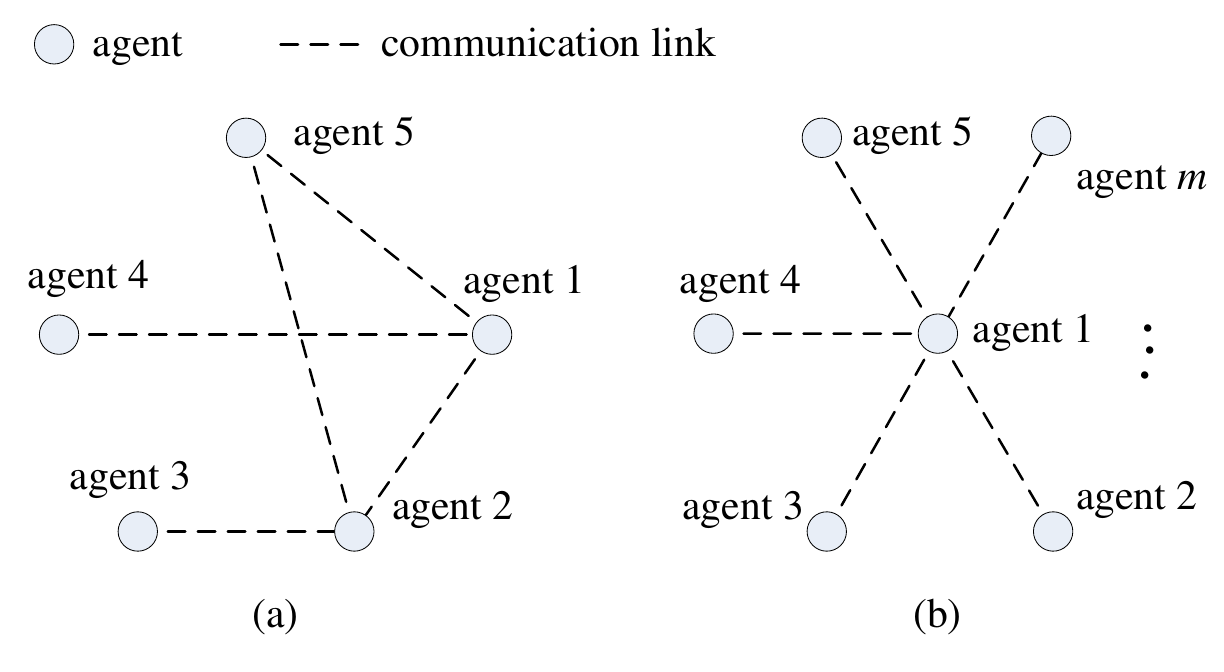}}
 		\caption{Decentralized network structures: (a) fully decentralized structure with 5 agents; (b) master-slave structure for $m$ agents.}
 		\label{11}
 	\end{center}
 	\vskip -0.2in
 \end{figure}

 \begin{figure} [t]    	
	\centering
	\subfigure[ ]{
		\includegraphics[width=4.12 cm]{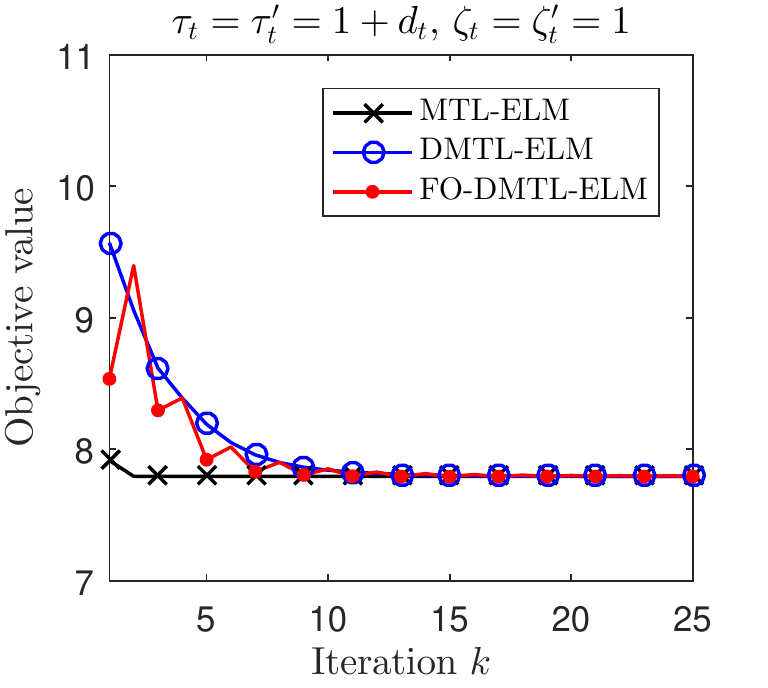}  		 
	}
	\subfigure[ ]{ 
		\includegraphics[width=4.12  cm]{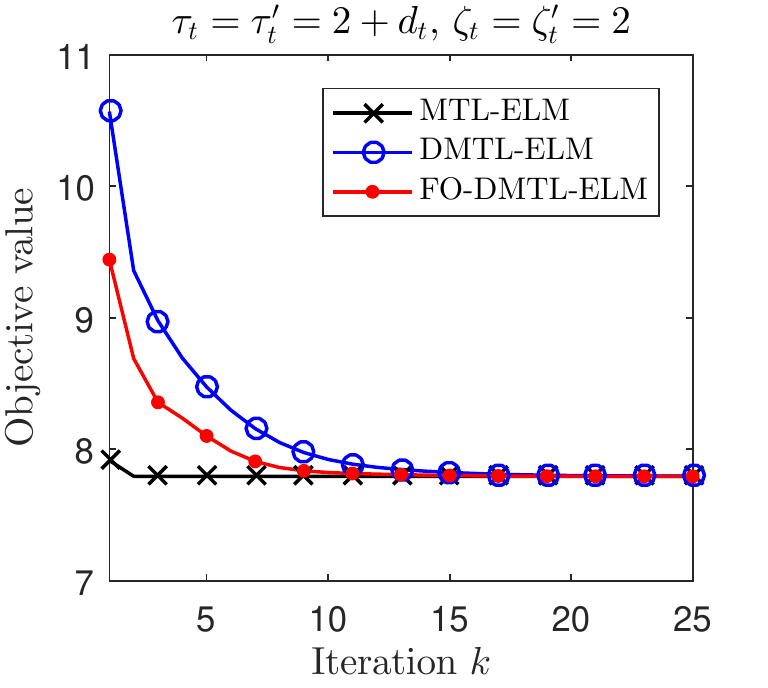}  	 
	}
\subfigure[ ]{
	\includegraphics[width=4.12  cm]{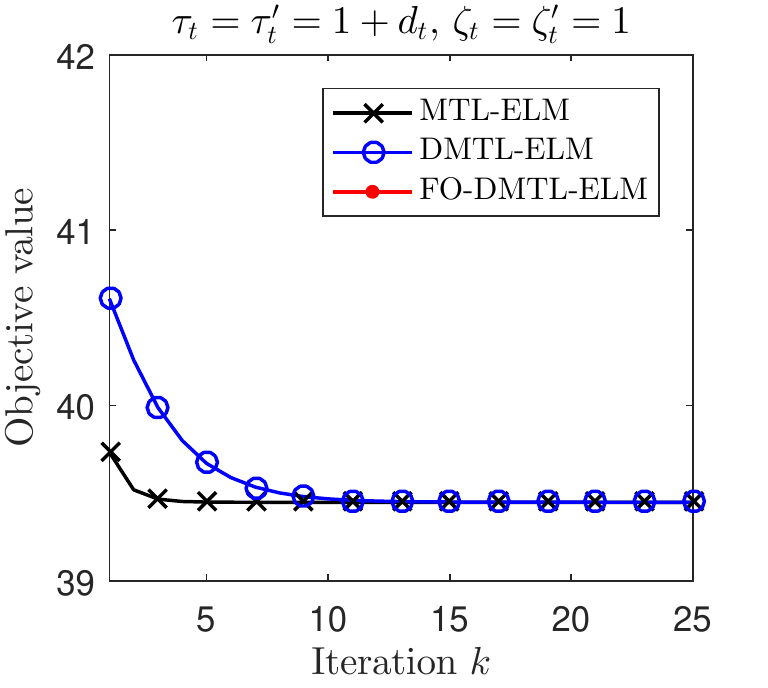}  		 
}
\subfigure[ ]{ 
	\includegraphics[width=4.12  cm]{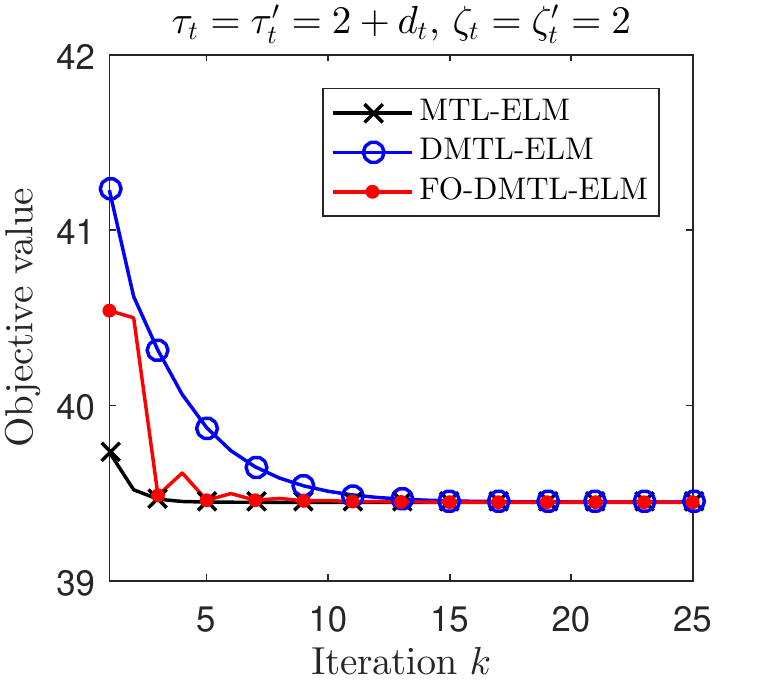}  	 
}
	\caption{Objective value with iterations for proposed methods MTL-ELM, DMTL-ELM and FO-DMTL-ELM.}
	\label{fig:3}
\end{figure}

 The convergence behavior of MTL-ELM, DMTL-ELM and FO-DMTL-ELM are presented in Fig. 3. 
 Since MTL-ELM is a centralized method, it has the fastest convergence speed. We denote $\tau_t,\zeta_t$ the parameters of DMTL-ELM while $\tau_t',\zeta_t'$ as those of FO-DMTL-ELM. For Fig. 3 (a) and (b) with $L=5,N_t=10,\rho=1,\delta=10$ and $\gamma_i^{k+1}=\min\{1, \frac{\delta \| \hat{C}_i\bm{U}^k-\hat{C}_i\bm{U}^{k+1}   \|^2 }{ \| \hat{C}_i\bm{U}^{k+1}  \|^2} \}$, the convergence happens with $\tau_t=\tau_t'\in\{1+d_t,2+d_t\}$ and $\zeta_t=\zeta_t'\in\{1, 2 \}$ since conditions in Theorems 1 and 2 are satisfied. However, comparing (a) with (b), the convergence speed for these two cases is different for both DMTL-ELM and FO-DMTL-ELM. 
 This is because larger parameters $\tau_t(\tau_t')$ and $\zeta_t(\zeta_t')$ reduce the update step of $U_t$ and $A_t$.
 Thus, we can trade off the convergence speed for looser convergence requirement by adjusting the parameters of algorithms DMTL-ELM and FO-DMTL-ELM. Note that there is jitter in the objective value for FO-DMTL-ELM method, which demonstrates that it is sensitive to the parameters $\tau_t'$ and $\zeta_t'$ when they are close to the requirements stated in Theorem 2. In Fig. 3 (b), the FO-DMTL-ELM decreases faster than DMTL-ELM due to the larger updating step introduced by first-order approximation.
 
 While in Fig. 3 (c) and (d), we evaluate proposed methods with $L=10$ and $N_t=100$. As Fig. 3 (c) shown, both MTL-ELM and DMTL-ELM can converge except FO-DMTL-ELM, where $\tau_t=\tau_t'=1+d_t$ and $\zeta_t=\zeta_t'=1$. But when the parameters increases, the FO-DMTL-ELM show convergence despite of jitter. This is because Lipschitz constant $L_t$ is enlarged by increasing dimension $L$. Therefore a larger $\tau_t'$ is required for FO-DMTL-ELM. This also supports the analysis in Theorem 2.

We also reveal the element evolution of $U_t$ and $A_t$ for DMTL-ELM and FO-DMTL-ELM in Fig. 4 (a) and (b), respectively, where $\tau_t= \tau_t'=1+d_t$, $\zeta_t=\zeta_t'=1$ and $L=5$, $N_t=10$. Since $U_t\in\mathbbm{R}^{10}$, we only compare the element $U_t^k(1,1) $ and ${U_t^k}'(1,1)$ with $U^{k=1000}(1,1)$ obtained by MTL-ELM algorithm, which is presented by the dashed line in Fig. 4 (a). From this figure, we can conclude that algorithm DMTL-ELM can ensure all the agents a unique subspace $U^*$. Moreover in Fig. 4 (b), the local predictor $A_t^k(1,1)$ and ${A_t^k}'(1,1)$ of tasks updated in DMTL-ELM and FO-DMTL-ELM algorithms converges to the corresponding $A_t^{k=1000}(1,1)$ from MTL-ELM with carrying out more iterations. Hence the optimal $A^*_t$ can also be guaranteed for each task. The evolution of ${U_t^k}'$ and ${A_t^k}'$ jitters with iterations. This is consistent with Fig. 3 (a). 
  
The accuracy of $U_t^k$ and $A_t^k$, which are defined as $(\frac{1}{mLr}\sum_{t=1}^{m} \|  U_t^k- U^{k=1000} \|^2 )^{1/2}$ and $(\frac{1}{mrd}\sum_{t=1}^{m} \|  A_t^k- A_t^{k=1000} \|^2)^{1/2}$, are presented in Fig. (c) and (d). The accuracy for FO-DMTL-ELM has a faster decrease speed at the beginning of iterations than DMTL-ELM. However, when $k$ grows, the DMTL-ELM can guarantee more accurate outputs.  
 
 \begin{figure}[t]   	
	\centering
	\subfigure[ ]{
		\includegraphics[width=4.1 cm]{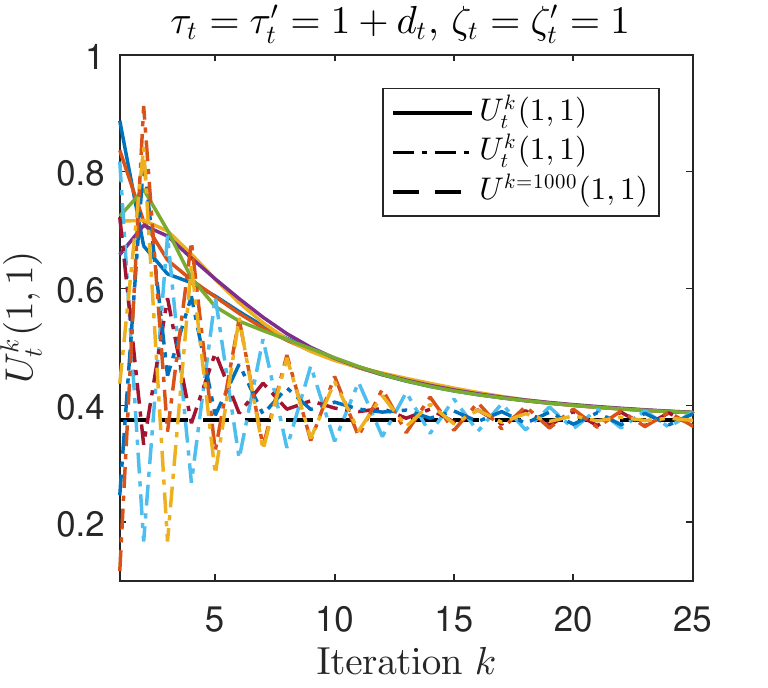}  		 
	}
	\subfigure[ ]{ 
		\includegraphics[width=4.1 cm]{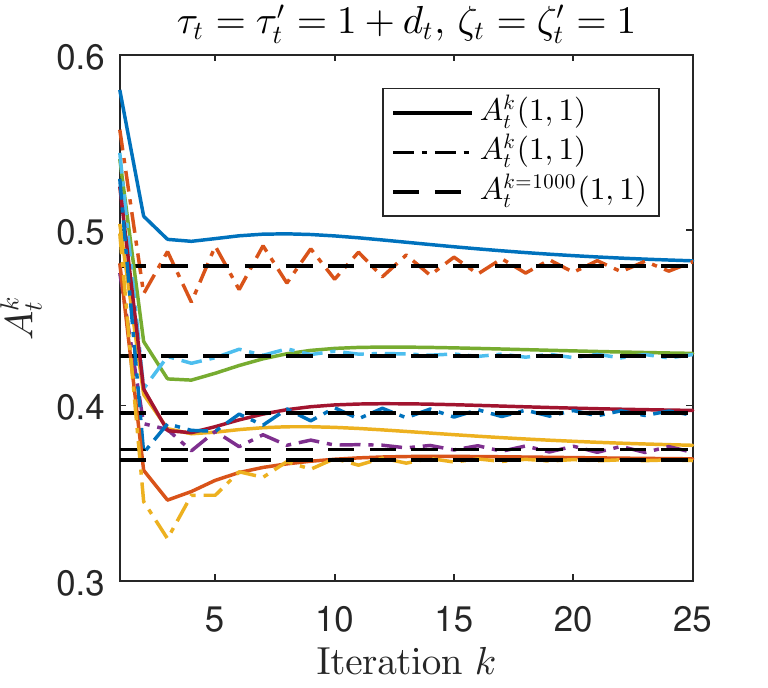}  	 
	}
	\subfigure[ ]{
		\includegraphics[width=4.1 cm]{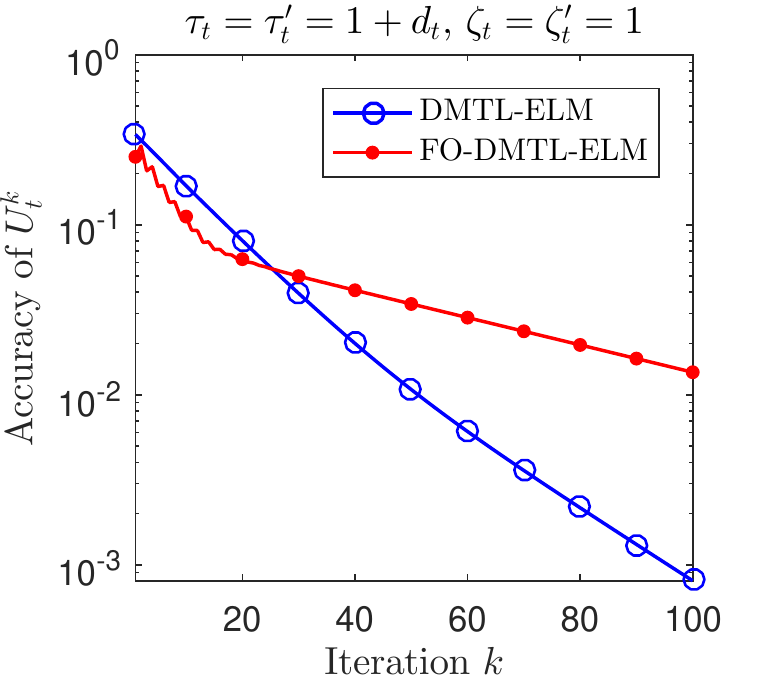}  		 
	}
	\subfigure[ ]{ 
		\includegraphics[width=4.1 cm]{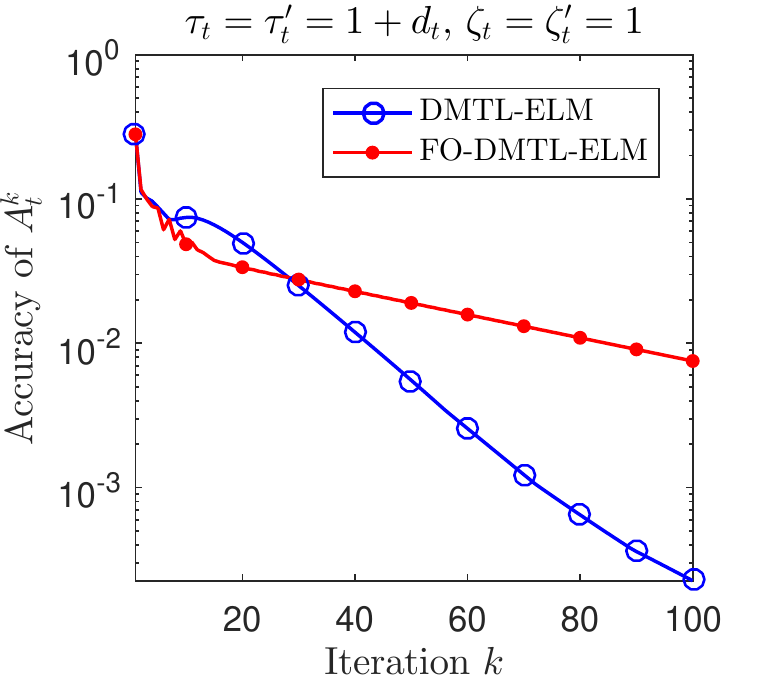}  	 
	}
	\caption{ For algorithms DMTL-ELM and FO-DMTL-ELM: evolution of $U_t^k$ and $A_t^k$; accuracy for $U_t^k$ and $A_t^k$.}
	\label{fig:4}
\end{figure}

\subsection{Generalization Performance}

We perform extensive studies to evaluate the proposed approaches by empirical comparison with the following baselines for centralized and distributed MTL, respectively:

\renewcommand{\arraystretch}{1.2}
\begin{table*}[tp]
	\small
	\centering 
	\fontsize{6.5}{8}\selectfont
	\caption{Comparison of testing error(\%) and running time (s) for different learning approaches and training data sets.}
	
	\label{tab:performance_comparison1}
	\scalebox{1}{
		\begin{tabular}
			{p{1cm}<{\centering} | p{0.5cm}<{\centering} |  p{0.5cm}<{\centering} || p{0.5cm}<{\centering} | p{0.5cm}<{\centering} |  p{0.5cm}<{\centering}|p{0.5cm}<{\centering}|p{0.5cm}<{\centering}|p{0.5cm}<{\centering}|| p{0.5cm}<{\centering} |  p{0.5cm}<{\centering} | p{0.5cm}<{\centering} | p{0.5cm}<{\centering} |  p{0.5cm}<{\centering}|p{0.5cm}<{\centering}|p{0.5cm}<{\centering}|p{0.5cm}<{\centering}  }

			\hline
			\multirow{2}{*}{Dataset}&
			\multicolumn{2}{c ||  }{Local ELM }&\multicolumn{2}{c|}{ MTFL}&\multicolumn{2}{c| }{ GO-MTL}&\multicolumn{2}{c||}{MTL-ELM}&\multicolumn{2}{c|}{DGSP}&\multicolumn{2}{c|}{ DNSP}&\multicolumn{2}{c|}{DMTL-ELM}&\multicolumn{2}{c  }{ FO-DMTL-ELM}\cr\cline{2-17}
			    &testing error & running time  &testing error & running time&testing error & running time&testing error & running time&testing error & running time&testing error & running time&testing error & running time&testing error & running time  \cr\hline
			USPS &4.26 &0.009 & 4.67 &0.10 &  6.30 &7.52 &  {\bf 3.49}&226.1 & 5.05 & 0.03&4.47 &0.04 &{\bf 3.54} & 184.2&3.89&22.5 \cr\hline 
 		
			MNIST &6.58 & 0.004& 6.84 &0.20 & 9.76  &8.10 &  {\bf 5.90} & 244.2& 7.9 & 0.04 &7.35 &0.07 &\textbf{5.96}& 192.6 &  6.20  & 19.7\cr\hline 
	\end{tabular}}
\end{table*}

\noindent
\textit{Separate approach}:
 
\noindent
\textbf{Local ELM}: a baseline single-task learning method utilizing ELM, in which the output weight of each task is learned separately by tasks;

\noindent
\textit{Centralized approaches}:

\noindent
\textbf{Multi-Task Feature Learning (MTFL) \cite{argyriou2008convex}}: an AO based algorithm solving the equivalent convex problem with respect to the weight $\bm{w}_t$ of tasks and the correlation matrix $\bm{\Omega}$;

\noindent
\textbf{Grouping and Overlap for Multi-Task learning (GO-MTL)\cite{kumar2012learning}}: a framework for multi-task learning assuming that each task parameter vector is a linear combination of a finite number of underlying basis tasks;
	
\noindent 
\textit{Distributed approaches}:

\noindent 
\textbf{Distributed Gradient Subspace Pursuit (DGSP)\cite{wang2016distributed2}}: a distributed approach aiming to learn an unknown shared low-dimensional subspace for related tasks. The subspace gets supplemented at all tasks in each iteration; 

\noindent
\textbf{Distributed Newton Subspace Pursuit (DNSP)\cite{wang2016distributed2}}: the same as approach DGSP except the way to pursuit the shared subspace, which substitutes gradient direction with Newton direction in iterations. 
 
Since the methods DGSP and DNSP can only work in the master-slave structure, we consider the decentralized setting for DMTL-ELM and FO-DMTL-ELM shown as Fig. 2 (b). 
 
\begin{figure} 
	\vskip 0.2in
	\begin{center}
		\centerline{\includegraphics[width=70mm]{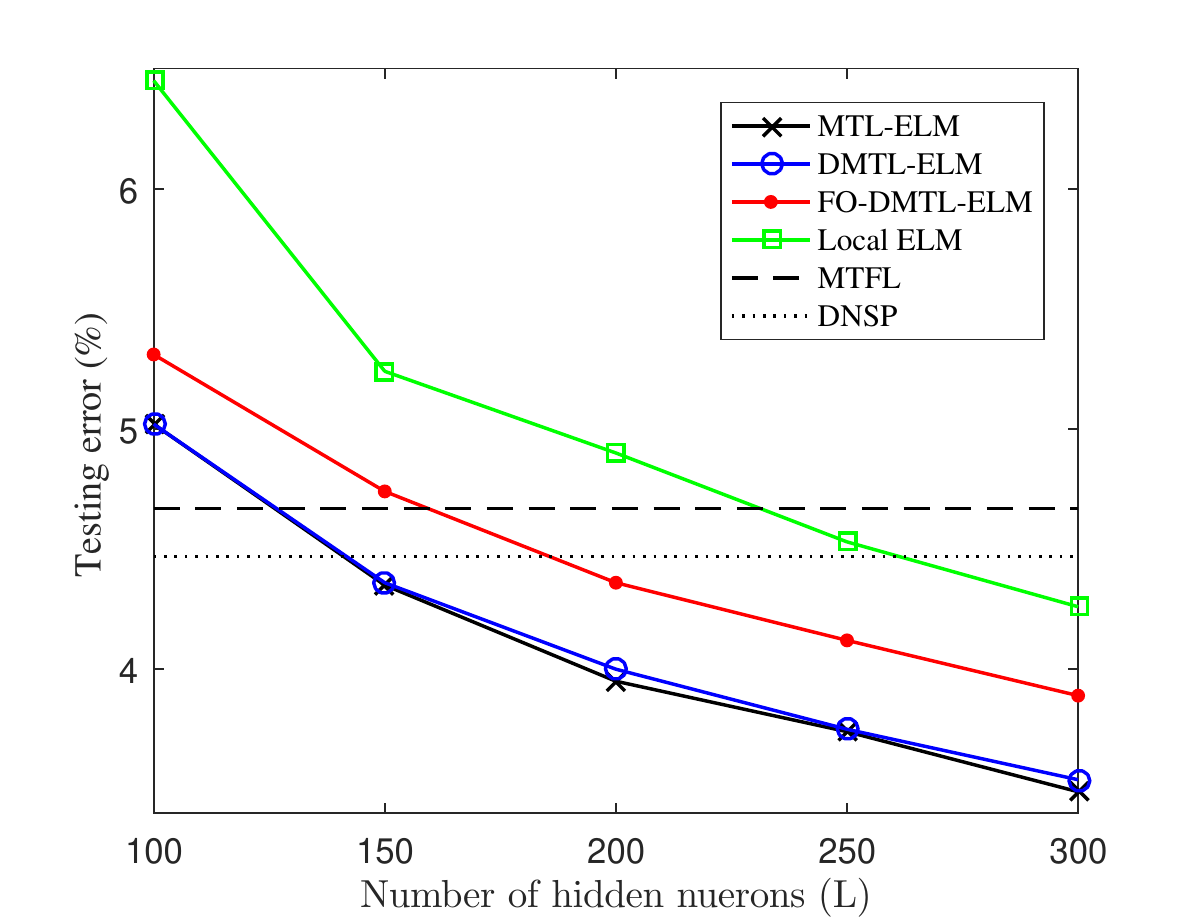}}
		\caption{Testing error comparison for Local ELM, MTL-ELM, DMTL-ELM and FO-DMTL-ELM with state-of-the-art methods over USPS dataset.}
		\label{icml-historical1}
	\end{center}
	\vskip -0.2in
\end{figure}  
 
Our goal is to compare the proposed centralized and decentralized multi-task learning methods with the corresponding baseline approaches. For completeness, we examine the proposed MTL-ELM, DMTL-ELM and FO-DMTL-ELM algorithms as well as state-of-the-art methods over the following two real-world datasets with classification tasks. According to experiments in \cite{huang2012extreme}, 
we adopt Sigmoid additive hidden node of which the activation function is
 \begin{equation}
  \vphantom{\frac{1}{2}}G(\bm{w},b,\bm{X}) = \frac{1}{1+ \exp(-(\bm{w}\cdot\bm{X} + b))}.
 \end{equation}

For algorithms MTL-ELM, DMTL-ELM and FO-DMTL-ELM, we run the iterations for $k=100$ times, respectively. Meanwhile in DMTL-ELM and FO-DMTL-ELM, we use \textit{Standard Proximal} $P_t$ and $Q_t$, and set parameters $\rho=1$, $\delta=100$, $\gamma_i^{k+1}=\min\{1,\frac{\delta \| \hat{C}_i(\bm{U}^k- \bm{U}^{k+1})   \|^2 }{  \| \hat{C}_i\bm{U}^{k+1}  \|^2}  \}$. We run simulations over the following data sets for 100 times and average the testing error and running time.

\noindent
\textbf{USPS digits dataset}:
This is a handwritten digits dataset \cite{kang2011learning} with 10 classes and $256$ input dimension. The images are precessed using principle component analysis (PCA) and dimensionality is reduced to 64 with retaining almost $95\%$ of the variance. 
We extract 1350 samples, 900 samples out of which are used for training while the other 450 samples are used for testing. We set the task number as $m=10$, where each task conducts classification over 3 random classes. Meanwhile the training and testing samples for each task are randomly and equivalently allocated. The regularization parameters of MTL-ELM, DMTL-ELM and FO-DMTL-ELM are chosen as $\lambda=\mu=\sqrt{10}$. We set $\tau_t=10+d_t$ and $\zeta_t=30$ for DMTL-ELM while $\tau_t'=20+d_t,\zeta_t'=30$ for FO-DMTL-ELM. The parameters for MTFL are set as $\gamma=10$ and $\epsilon=20$, while the regularization parameter for other methods is set as $\lambda=10$.

\begin{figure}[t]
	\vskip 0.2in
	\begin{center}
		\centerline{\includegraphics[width=70mm]{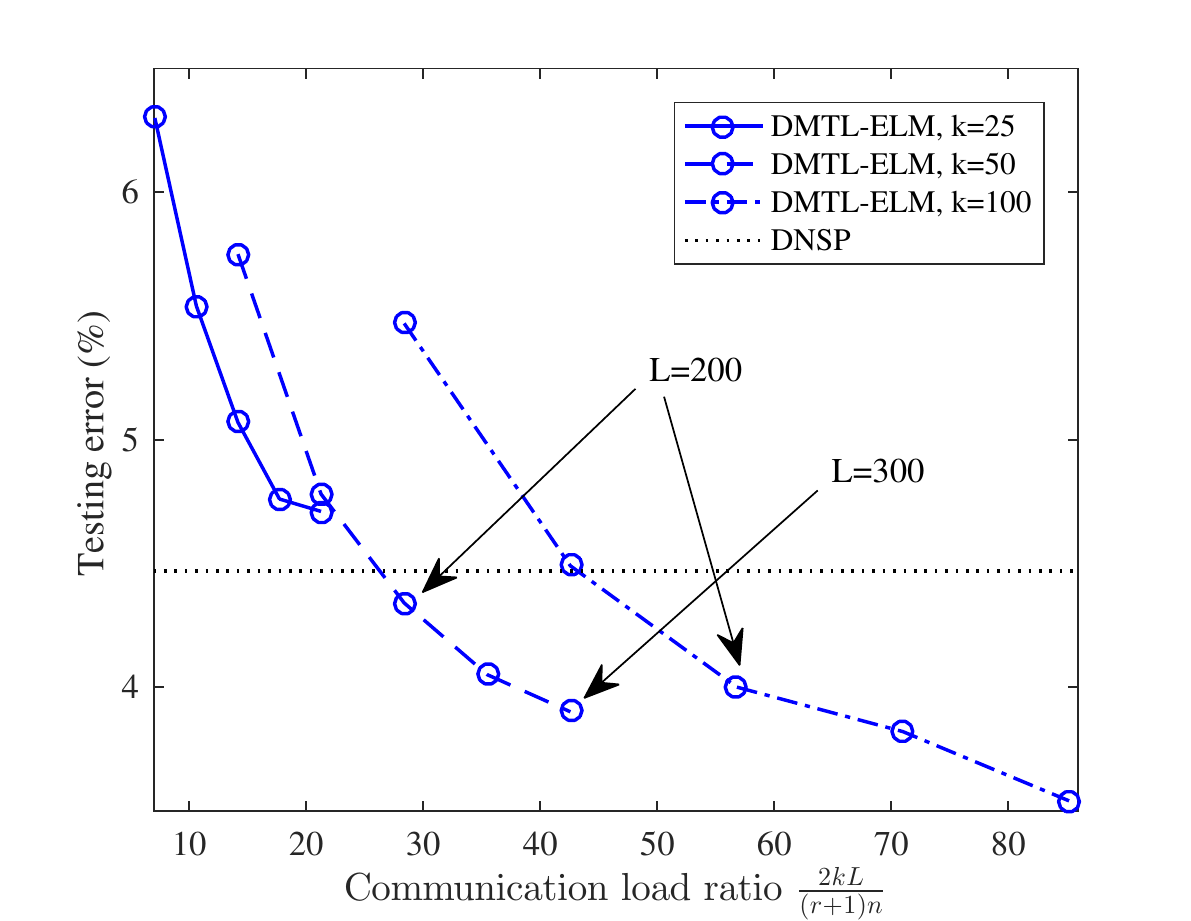}}
		\caption{Testing error of DMTL-ELM with the communication ratio of DMTL-ELM to DNSP over the USPS dataset.}
		\label{icml-histo 2}
	\end{center}
	\vskip -0.2in
\end{figure}  

The generalization performance for proposed methods is presented in Fig. 5. By sharing information across tasks to learn output weights, the MTL-ELM, DMTL-ELM and FO-DMTL-ELM schemes can achieve smaller testing error than the separate learning method Local ELM. 
Since DMTL-ELM can converge to the outputs of MTL-ELM with $k=100$, the DMTL-ELM has ignorable performance loss over testing error compared with MTL-ELM. Due to the first-order approximation, the outputs of FO-DMTL-ELM does not approach the optimum with $k=100$ iterations, the testing error performance decays from DMTL-ELM but still outperform Local-ELM method.   
From Fig. 5, the testing error of MTL-ELM, DMTL-ELM with $L\geq150$ and FO-DMTL-ELM with $L\geq 200$ are smaller than that of MTFL and DNSP methods. Since the dimension of $A_t$ is irrelevant to $L$, the dimension of shared subspace $U$ will be expanded with $L$, which separates the latent tasks more easily.  
Hence the generalization performance for all the methods based on ELM can be improved. The testing errors of Local ELM, MTL-ELM and DMTL-ELM and FO-DMTL-ELM in Table 1 are obtained with $L=300$, where $\tau_t=20+d_t, \zeta_t=40$ for DMTL-ELM and $\tau_t'=30+d_t, \zeta_t'=40$ for FO-DMTL-ELM.

\noindent
\textbf{MNIST digits dataset}:
This is also a digit dataset \cite{kang2011learning} with 10 classes and $784$ input dimensions. We reduce the dimensions of images to 87 by using PCA.
We extract 1350 samples, 900 samples out of which are used for training while the other 450 samples are used for testing. We set the task number as $m=10$, where each task conducts classification over 3 random classes. The training and testing samples for each task are randomly and equivalently allocated.
The regularization parameters of MTL-ELM, DMTL-ELM and FO-DMTL-ELM are chosen as $\lambda=\mu=\sqrt{20}$, while $\lambda=100$ for GO-MTL. We set the parameters for MTFL as $\gamma=10$ and $\epsilon=20$, and the regularization parameter for other methods is set as $\lambda=20$. To obtain the testing errors given in Table 1, we set the number of hidden neurons as $L=300$ for Local ELM, MTL-ELM and DMTL-ELM methods, and $\tau_t=20+d_t, \zeta_t=40$ for DMTL-ELM and $\tau_t'=30+d_t, \zeta_t'=40$ for FO-DMTL-ELM.

Table 1 demonstrates that the testing error achieved by MTL-ELM, DMTL-ELM and FO-DMTL-ELM can outperform other MTL methods over the tested data sets. Meanwhile, the proposed DMTL-ELM algorithm can achieve almost the same testing error compared with centralized MTL-ELM. Though with decayed generalization performance with same iteration $k$, the running time of FO-DMTL-ELM is much faster than MTL-ELM and DMTL-ELM. Among all the tested methods, the Loal ELM has the shortest running time.

\subsection{Communication load of DMTL}

The communication load of distributed learning approaches is important in real-word applications, since too much information exchanged can cause system overload. Since the communication load of DMTL-ELM and FO-DMTL-ELM are the same when the iterations $k$ is identical. Then we only evaluate the communication loads and testing error of the proposed DMTL-ELM algorithm compared with the DMTL method DNSP, which can achieve better generalization performance than DGSP. In each iteration of the DMTL-ELM algorithm, the information that agent $t$ broadcasts to its neighbors is only the updated $U_t$. Hence while the decentralized network structure is fixed, the communication load of DMTL-ELM is determined by the number of iterations $k$ and dimension of $U_t$, which is the number of hidden neurons $L$. In each iteration of the DNSP method, which works in the master-slave network structure, the updated $U_t$ in slaves are first sent to the master. Then a new column $u$ of the desired subspace $U$ is calculated and broadcast to slaves. But the difference between DNSP and DMTL-ELM method is that the iterations of the former is only $r$, which is the number of latent tasks, instead of the $k$ in the latter one. Thus for the master-slave structure, the communication load of DMTL-ELM is $\frac{2kL}{(r+1)n}$ times that of DNSP.

In Fig. 6, we present the testing error of DMTL-ELM algorithm with respect to the communication load ratio of DMTL-ELM to DNSP method when using USPS dataset, where $L\in\{100,150,200,250,300\}$ and $k\in\{25,50,100\}$.  
Since the number of latent tasks is identical, increasing $k$ and $L$ will enlarge the communication load of DMTL-ELM, as well as guarantee a smaller testing error. 
Thus we can conclude that there exists a trade-off between the generalization performance of DMTL-ELM and the communication load. For the considered regions of $k$ and $L$ in Fig. 6, the communication load of DNSP is smaller than that of DMTL-ELM since the ratio is larger than $1$. When we fix the iteration rounds as $k=25$ of DMTL-ELM, the testing error of DNSP is always smaller than that of DMTL-ELM even rough increasing the hidden neurons $L$ in considered region. This is because the outputs $\{U_t\}$ and $\{A_t\}$ of DMTL-ELM with $k=25$ are far away from the optimal value. However, when $k$ increase to $50$ and $100$, the DMTL-ELM can achieve better generalization performance than DNSP by enlarging $L$, with larger communication load than DNSP. Moreover from Fig. 6, when we set $k=50,L=300$,
the generalization performance is better than that of $k=50,L=200$, but with a smaller communication cost. Thus, in real applications, we can choose appropriate values of $k$ and $L$ to balance the communication load and generalization performance of DMTL-ELM.

\section{Conclusions}

We study the multi-task learning problem based on an ELM implementation. We first introduce the centralized MTL problem based on ELM, and present the AO-based algorithm MTL-ELM to solve the problem. Then we extend to the decentralized scenario by formulating the DMTL problem as a majorized multi-block optimization with coupled bi-convex objective functions. It is solved by our proposed algorithm DMTL-ELM, which is a hybrid Jacobian and Gauss-Seidel ADMM. The computation load of DMTL-ELM is further reduced by introducing first-order approximation in the proposed FO-DMTL-ELM algorithm. Through analysis we prove the convergence of MTL-ELM, DMTL-ELM and FO-DMTL-ELM, as well as presenting the required conditions for parameters. Simulations verified the convergence of all the approaches, and show that the generalization performance of MTL-ELM, DMTL-ELM and FO-DMTL-ELM can outperform state-of-the-art MTL methods by randomly mapping the input of each task to a hidden feature space and learning the output weights together across tasks. Moreover through adjusting the number of hidden neurons of ELM and iteration rounds, we present the trade-off between communication load and the generalization performance of DMTL-ELM algorithm.

\bibliography{ref}
\bibliographystyle{IEEEtran}

\end{document}